\documentclass[onefignum,onetabnum,onealgnum]{siamonline190516}

\usepackage{lipsum}

\usepackage{endnotes}

\usepackage{titlesec}
\usepackage{titletoc}

\titleclass{\task}{straight}[\section]
\newcounter{task}
\renewcommand{\thetask}{\arabic{task}}
\titleformat{\task}[hang]
    {\normalfont\LARGE\bfseries}{Task \thetask:}{1em}{}

\titleformat*{\task}{\color{header1}\bfseries}

\titlecontents{task}
              [3.8em] 
              {}
              {\contentslabel{2.3em}}
              {\hspace*{-2.3em}}
              {\titlerule*[1pc]{.}\contentspage}

\titlespacing*{\section}{0ex}{1ex}{1ex}
\titlespacing*{\subsection}{0ex}{1ex}{1ex}
\titlespacing*{\paragraph}{0ex}{1ex}{1ex}
\titlespacing*{\subparagraph}{0pt}{1ex}{1ex}
\titlespacing*{\task}{0em}{1ex}{1ex}

\usepackage[sort&compress,comma,square,numbers]{natbib}

\usepackage{amsfonts,amsmath,amssymb}
\usepackage{bbm}
\usepackage{xspace}

\usepackage{multicol}
\usepackage{enumitem}

\usepackage{hhline}

\usepackage{comment}
\usepackage{todonotes}
\RequirePackage{ulem}

\usepackage{booktabs}
\usepackage{titlesec}
\usepackage{fancyhdr}
\usepackage{fullpage}

\RequirePackage{titlesec}
\RequirePackage[font={small},{singlelinecheck=false}]{caption}
\usepackage[T1]{fontenc}
\usepackage[scaled]{helvet}
\usepackage[scaled=1.10, med]{zlmtt}

\usepackage{algorithm}
\usepackage{algpseudocode}

\providecommand{\sct}[1]{{\sc \texttt{#1}}}

\newcommand{\Dcov}{\sct{Dcov}}

\newcommand{\trace}[1]{{\ensuremath{\operatorname{trace}\!\left(#1\right)}}}           
\providecommand{\mc}[1]{\mathcal{#1}}
\providecommand{\mb}[1]{\boldsymbol{#1}}

\newcommand{\EE}{\mathbb{E}}           

\newcommand{\mbx}{\ensuremath{X}}

\newcommand{\mby}{\ensuremath{Y}}

\usepackage[super]{nth}
\providecommand{\sct}[1]{{\sc \texttt{#1}}}
\newcommand{\Dcor}{\sct{Dcor}}

\makeatletter
\def\thanks#1{\protected@xdef\@thanks{\@thanks
        \protect\footnotetext{#1}}}
\makeatother

\title{\bf The Chi-Square Test of Distance Correlation}
\author{Cencheng Shen$^{1,\dagger}$,
Sambit Panda$^2$, and
Joshua T.~Vogelstein$^{2,3}$
\thanks{
    $\dagger$ Corresponding author: \href{mailto:shenc@udel.edu}{shenc@udel.edu}
    $^1$ Department of Applied Economics and Statistics, University of Delaware
    $^2$ Institute for Computational Medicine, Department of Biomedical Engineering, Johns Hopkins University
    $^3$ Center for Imaging Science, Kavli~Neuroscience~Discovery Institute, Johns Hopkins University
    }
}

\begin{document}
\maketitle

\begin{abstract}
Distance correlation has gained much recent attention in the data science community: the sample statistic is straightforward to compute and asymptotically equals zero if and only if independence, making it an ideal choice to discover any type of dependency structure given sufficient sample size. One major bottleneck is the testing process: because the null distribution of distance correlation depends on the underlying random variables and metric choice, it typically requires a permutation test to estimate the null and compute the p-value, which is very costly for large amount of data. To overcome the difficulty, in this paper we propose a chi-square test for distance correlation. Method-wise, the chi-square test is non-parametric, extremely fast, and applicable to bias-corrected distance correlation using any strong negative type metric or characteristic kernel. The test exhibits a similar testing power as the standard permutation test, and can be utilized for K-sample and partial testing. Theory-wise, we show that the underlying chi-square distribution well approximates and dominates the limiting null distribution in upper tail, prove the chi-square test can be valid and universally consistent for testing independence, and establish a testing power inequality with respect to the permutation test.
\end{abstract}

\noindent%

\section{Introduction}
\label{sec:intro}

Given pairs of observations $(x_{i},y_{i}) \in \mathbb{R}^{p} \times \mathbb{R}^{q}$ for $i=1,\ldots,n$, assume they are independently identically distributed as $F_{\mbx \mby}$. Two random variables are independent if and only if the joint distribution equals the product of the marginal distributions. The statistical hypothesis for testing independence is
\begin{align*}
& H_{0}: F_{\mbx \mby}=F_{\mbx}F_{\mby},\\
& H_{A}: F_{\mbx \mby} \neq F_{\mbx}F_{\mby}.
\end{align*}
Detecting the potential relationships underlying the sample data has long been a fundamental question in theoretical and applied research. The traditional Pearson correlation \citep{Pearson1895} has been a valuable tool in quantifying the linear association and applied in many branches of statistics and machine learning. To detect all types of dependence structures, a number of universally consistent methods have been proposed recently, such as the distance correlation \citep{SzekelyRizzoBakirov2007, SzekelyRizzo2009}, the Hilbert-Schmidt independence criterion \citep{GrettonEtAl2005,GrettonGyorfi2010}, the Heller-Heller-Gorfine statistics \citep{HellerGorfine2013,heller2016consistent}, the multiscale graph correlation \citep{mgc1, mgc2}, among many others. The Hilbert-Schmidt independence criterion (HSIC) can be thought of as a kernel version of distance correlation and vice versa \citep{SejdinovicEtAl2013,exact2}, and the multiscale graph correlation is an optimal local version of distance correlation. These universally consistent dependence measures have been applied to two-sample testing \citep{RizzoSzekely2016, exact1}, conditional and partial testing \citep{gretton2007,SzekelyRizzo2014,Wang2015}, feature screening \citep{LiZhongZhu2012, bal2013, Zhong2015,mgc4}, clustering \citep{Szekely2005,Rizzo2010}, time-series \citep{Zhou2012,Pitsillou2018,mgc5}, graph testing \citep{mgc2,mgc6}, etc.

To populate these methods to big data analysis, a big hurdle is the time complexity. Computing the distance correlation typically requires $\mc{O}(n^2)$; and to compute its p-value for testing, the standard approach is to estimate the null distribution of distance correlation via permutation, which requires $\mc{O}(r n^2)$ where $r$ is the number of random permutations and typically at least $100$ or more. In comparison, for one-dimensional ($p=q=1$) data the Pearson correlation can be computed in $\mc{O}(n)$, and a Pearson correlation t-test is readily available to compute the p-value in constant time complexity $\mc{O}(1)$. The computational advantage makes Pearson extremely fast and attractive in practice. Recent works have successfully expedited the distance correlation computation into $\mc{O}(n \log(n))$ under one-dimensional data and Euclidean distance \citep{Huo2016,Hu2018}. The testing part, however, remains difficult and less well-understood, because the null distribution of distance correlation has no fixed nor known density in general. Some notable works include the distance correlation t-test \citep{SzekelyRizzo2013a}, HSIC Gamma test \citep{GrettonGyorfi2010}, and the subsampling approach \citep{zhang2018}, which provided special treatments and very valuable insights into the problem. 

In this paper, we propose a chi-square test for the bias-corrected distance correlation. It is simple and straightforward to use, has a constant time complexity without the need to permute nor subsampling nor parameter estimation, and does not rely on any distributional assumption on data. In particular, the bias-corrected distance correlation can be computed in $\mc{O}(n \log n)$ under one-dimensional data and Euclidean distance, which renders the method comparable in speed to the Pearson correlation t-test and scalable to billions of observations. The test is applicable to any multivariate data, any strong negative type metric or characteristic kernel, and can be utilized for K-sample and partial testing.

Theory-wise, we prove the chi-square test is universally consistent and valid for sufficiently large $n$ and sufficiently small type 1 error level $\alpha$ (generally $n \geq 20$ and $\alpha \leq 0.05$ suffice), has a similar testing power as the permutation test, and is the most powerful among all valid tests of distance correlation using known distributions (i.e., any test that is based on a fixed distribution, such as using t-test or F-test). In particular, we prove that the underlying chi-square distribution can well-approximate and dominate the limiting null distribution of the bias-corrected distance correlation in upper tail, and establish a testing power inequality among the chi-square test, the distance correlation t-test from \citet{SzekelyRizzo2013a}, and the standard permutation test. 

The advantages of the chi-square test are supported by simulations and real data experiments. The code are openly available on Github and implemented in Matlab\footnote{\url{https://github.com/neurodata/mgc-matlab}} and Python\footnote{\url{https://github.com/neurodata/mgc}}. The Appendix includes detailed background information, all theorem proofs and intermediate results, and detailed simulation functions.

\section{Method}
\label{sec:main}
\subsection{The Distance Correlation Chi-Square Test}

The proposed chi-square test is stated in Algorithm~\ref{alg1}: given paired sample data $(\mathbf{X},\mathbf{Y}) \in \mathbb{R}^{n \times (p+q)}$, first compute the bias-corrected sample distance correlation $C=\Dcor_{n}(\mathbf{X},\mathbf{Y})$, then take the p-value as $p=Prob(nC < \chi^{2}_{1} -1)$. Then at any pre-specified type 1 error level $\alpha$, the independence hypothesis is rejected if and only if $p < \alpha$. Unless mentioned otherwise, in this paper distance correlation always means the bias-corrected sample distance correlation (see Appendix Section~\ref{sec:bg} for the algebraic expression).

As the chi-square distribution is standard in every software package, the p-value computation takes $\mc{O}(1)$ regardless of sample size, which is much faster than the standard permutation approach (see Appendix Algorithm~\ref{alg2}) requiring $\mc{O}(r n^2)$. The statistic computation and thus Algorithm~\ref{alg1} on its whole require $\mc{O}(n^2)$ in general. 

\begin{algorithm}
\caption{The Distance Correlation Chi-Square Test for Independence}
\label{alg1}
\begin{algorithmic}
\Require Paired sample data $(\mathbf{X},\mathbf{Y})=\{(x_i,y_i) \in \mathbb{R}^{p+q} \mbox{ for }  i \in [n]\}$.
\Ensure The bias-corrected distance correlation $C$ and its p-value $p$.
\Function{FastTest}{$\mathbf{X},\mathbf{Y}$}
\State $C=\Dcor_{n}(\mathbf{X},\mathbf{Y})$; \Comment{the bias-corrected distance correlation}
\State $p=1-F_{\chi^{2}_{1}-1}(n \cdot C)$; \Comment{reject the null when $p<\alpha$}
\EndFunction
\end{algorithmic}
\end{algorithm}

The chi-square test is well behaved from the following theorem (which follows directly from Theorem~\ref{cor1} in Section~\ref{sec:main1}):
\begin{theorem}
\label{thm7}
The distance correlation chi-square test that rejects independence if and only if 
\begin{align*}
n \Dcor_{n}(\mathbf{X},\mathbf{Y}) \geq F_{\chi^{2}_{1}-1}^{-1}(1-\alpha)
\end{align*}
is a valid and universally consistent test for sufficiently large $n$ and sufficiently small type 1 error level $\alpha$.
\end{theorem}
In practice, $n \geq 20$ suffices, and the validity empirically holds for any $\alpha \leq 0.05$. 

\subsection{Fast Statistic Computation}
\label{sec:ext}
In the special case of $p=q=1$ and Euclidean distance, Algorithm~\ref{alg1} can run significantly faster:

\begin{theorem}
\label{thm8}
Suppose $p=q=1$ and we use the Euclidean distance for the bias-corrected distance correlation. Then Algorithm~\ref{alg1} can be implemented with a computational complexity of $\mc{O}(n \log n)$. 
\end{theorem}

Essentially the bias-corrected sample distance correlation can be computed in $\mc{O}(n \log n)$ based on the results from \citet{Huo2016} and \citet{Hu2018}. This makes one-dimensional testing comparable in speed to the Pearson correlation t-test. On a standard Windows 10 machine using MATLAB, we are able to test independence between a million pairs of observations ($p=q=1$) within $10$ seconds with space requirement of $\mc{O}(n)$. Previously, the statistic computation and the permutation test need a space complexity of $\mc{O}(n^2)$ and a time complexity of $\mc{O}(r n^2)$, which would have required external disk storage and days of computation to finish testing the same amount of data. The implementation details can be found in Appendix and Github Matlab code.

\subsection{Chi-Square Test for K-Sample}
The chi-square test is readily applicable to any inference task using distance correlation, or any statistic involving a similar trace operation with bias-corrected matrix modification. Two immediate extensions are the K-sample test and partial test. Previously, the permutation test has been the standard approach for both of them.

Given $K$ sets of sample data $\mathbf{U}^{k}=[u_{1}^{k}|u_{2}^{k}|\ldots|u_{n_{k}}^{k}] \in \mathbb{R}^{p \times n_k}$ for $k=1,\ldots,K$, denote $\sum_{k=1}^{K} n_{k}=n$. The K-sample testing problem assumes $u_i^{k}$ are independently and identically distributed as $F_{U^{k}}$ for each $i$ and $k$, and aim to test
\begin{align*}
&H_{0} : F_{U^{1}} = F_{U^{2}} = \cdots = F_{U^{K}}, \\
&H_{A} : \mbox{ there exists at least one $F_{U^{k}}$ that is different from other distributions.}
\end{align*}
Algorithm~\ref{alg3} shows the K-sample variant of chi-square test.

\begin{algorithm}
\caption{The Distance Correlation Chi-Square Test for K-sample}
\label{alg3}
\begin{algorithmic}
\Require Sample data $\{\mathbf{U}^{k} \in \mathbb{R}^{p \times n_k} \mbox{ for }  k=1,\ldots,K\}$.
\Ensure The distance correlation $C$ and its p-value $p$.
\Function{FastKSample}{$\{\mathbf{U}^{k}\}$}
\Statex{\textbf{(1)} Data Transformation:} 
\State $\mathbf{X}=[\ \mathbf{U}^{1} \ |\ \mathbf{U}^{2} \ | \cdots | \ \mathbf{U}^{K} \ ]$; \Comment{concatenate Data}
\State $\mathbf{Y}=zeros(K,N)$; \Comment{a zero matrix}
\For{$k=1,\ldots,K$}
\For{$i=1,\ldots,n_k$}
\State $\mathbf{Y}(k,\sum_{j=1}^{k-1}n_j + i)=1$;
\EndFor
\EndFor

\Statex{\textbf{(2)} Test:}
\State $C=\sct{Dcor}_{n}(\mathbf{X},\mathbf{Y})$; 
\State $p=1-F_{\chi^{2}_{1}-1}(n \cdot C)$; \Comment{reject the null when $p<\alpha$}
\EndFunction
\end{algorithmic}
\end{algorithm}

It is shown in \citet{exact1} that by concatenating the sample data into $\mathbf{X}$ and forming an indicator matrix $\mathbf{Y}$ via one-hot encoding, $\Dcor_{n}(\mathbf{X}, \mathbf{Y}) \rightarrow 0$ if and only if the null hypothesis is true. Then the next corollary follows from Theorem~\ref{thm7}.

\begin{corollary}
\label{cor3}
For sufficiently large $n$ and sufficiently small type 1 error level $\alpha$ and any $K \geq 2$, Algorithm~\ref{alg3} is valid and universally consistent for testing $F_{U^{1}} = F_{U^{2}} = \cdots = F_{U^{K}}$.
\end{corollary}

\subsection{Chi-Square Test for Partial}
Another application is to test whether the partial distance correlation equals $0$ or not. Given three sample data $\mathbf{X},\mathbf{Y},\mathbf{Z}$ of same sample size $n$, the partial distance correlation $\sct{PDcor}_{n}(\mathbf{X},\mathbf{Y};\mathbf{Z})$ and its population version $\sct{PDcor}(X,Y;Z)$ are defined in \citet{SzekelyRizzo2014}. Algorithm~\ref{alg4} shows the chi-square test for fast partial testing,

\begin{algorithm}
\caption{The Distance Correlation Chi-Square Test for Partial}
\label{alg4}
\begin{algorithmic}
\Require Paired sample data $(\mathbf{X},\mathbf{Y}, \mathbf{Z})=\{(x_i,y_i,z_i) \in \mathbb{R}^{p+q+s} \mbox{ for }  i \in [n]\}$.
\Ensure The partial distance correlation $C$ and its p-value $p$.
\Function{FastPartial}{$\mathbf{X},\mathbf{Y}, \mathbf{Z}$}
\State $C=\sct{PDcor}_{n}(\mathbf{X},\mathbf{Y};\mathbf{Z})$; \Comment{implementation details in \cite{SzekelyRizzo2014} or Github Matlab code}
\State $p=1-F_{\chi^{2}_{1}-1}(n \cdot C)$;\Comment{reject the null when $p<\alpha$}
\EndFunction
\end{algorithmic}
\end{algorithm}

The partial statistic shares the same trace formulation as bias-corrected distance correlation and operates on the same bias-corrected matrix modification, except using projected distance matrices rather than Euclidean distance matrices. Therefore we have the following result:
\begin{corollary}
\label{cor4}
For sufficiently large $n$ and sufficiently small type 1 error level $\alpha$, Algorithm~\ref{alg4} is valid and consistent for testing $\sct{PDcor}(X,Y;Z)=0$.
\end{corollary}

\section{Supporting Theory}
\label{sec:main1}

In this section we show the theory behind the chi-square test. Note that the results hold for any strong negative type metric or any characteristic kernel for $X$ and $Y$, e.g., Euclidean distance, Gaussian kernel, Laplacian kernel, etc., which means the chi-square test is also applicable to the Hilbert-Schmidt Independence Criterion. Some results (like Theorem~\ref{thm1} and Theorem~\ref{cor1}) are stated in limit or for sufficiently large sample size, for which $n \geq 20$ generally suffices for the limiting null distribution to be almost the same as the actual null.

\subsection{The Limiting Null Distribution and the Centered Chi-Square Distribution}

\begin{theorem}
\label{thm1}
The limiting null distribution of the bias-corrected distance correlation satisfies
\begin{align*}
n\Dcor_{n}(\mathbf{X},\mathbf{Y}) &\stackrel{D}{\rightarrow} \sum\limits_{i,j=1}^{\infty} w_{ij} (\mc{N}_{ij}^{2}-1),
\end{align*}
where the weights satisfy $w_{ij} \in [0,1]$ and $\sum\limits_{i,j=1}^{\infty} w_{ij}^{2} = 1$, and $\mc{N}_{ij}$ are independent standard normal distribution.
\end{theorem}

This theorem follows from \citet{zhang2018}. For different choice of metric and different marginal distributions, the weights $\{w_{ij}\}$ and the limiting null distribution are different, which is the main obstacle to design a test using known distribution. Nevertheless, from Theorem~\ref{thm1} we observe the mean and variance of $n\Dcor_{n}(\mathbf{X},\mathbf{Y})$ are always fixed and equal the mean and variance of $\chi^{2}_{1}-1$, which we shall call the centered chi-square distribution and denote by $U$ from now on.

\subsection{Upper Tail Dominance and Distribution Bounds}

We aim to show that $U \sim \chi^{2}_{1}-1$ dominates the limiting null distribution at some upper tail probability $\alpha$. We denote $F_{V}(x)$ as the cumulative distribution function of random variable $V$ at argument $x$, $F^{-1}_{V}(1-\alpha)$ as the inverse cumulative distribution function of random variable $V$ at probability $1-\alpha$, and formally define upper tail stochastic dominance as follows:
\begin{definition}
Given two random variables $U$ and $V$, we say $U$ dominates $V$ in upper tail at probability level $\alpha$ if and only if
\begin{align*}
F_{V}(x) \geq F_{U}(x)
\end{align*}
for all $x \geq F_{U}^{-1}(1-\alpha)$. This is denoted by 
\begin{align*}
V \preceq_{\alpha} U.
\end{align*}
\end{definition}

The next theorem plays a key role in proving the remaining theorems:
\begin{theorem}
\label{thm3}
Assume $U, U_1, U_2, \ldots, U_m$ are independently and identically distributed as $\chi_{1}^{2}-1$, and the weights $\{w_i \in [0,1], i=1,\ldots,m\}$ are decreasingly ordered and satisfies $\sum\limits_{i=1}^{m} w_{i}^{2}=1$. The summation density satisfies
\begin{align*}
f_{\sum\limits_{i=1}^{m}w_{i} U_i}(x) = O(e^{-x c/2})
\end{align*}
where the constant $c=\frac{1}{w_{1}} \in [1, \sqrt{m}]$.
\end{theorem}

Namely, the summation density is determined by the largest weight, which equals $U$ if and only if $w_1=1$. We can then bound the upper tail distribution of bias-corrected distance correlation as follows:
\begin{theorem}
\label{cor1}
For sufficiently large $n$, there exists $\alpha>0$ such that
\begin{align*}
\mc{N}(0,2) \preceq_{\alpha} n\Dcor_{n}(\mathbf{X},\mathbf{Y}) \preceq_{\alpha} U
\end{align*}
regardless of the metric choice or marginal distributions.
\end{theorem}
Therefore, despite the null distribution is subject to change without a fixed nor known density form, the centered chi-square distribution is always a valid approximation choice. Theorem~\ref{thm7} in Section~\ref{sec:main} (the validity and consistency of the chi-square test) follows trivially from Theorem~\ref{cor1}.

\subsection{Validity Level and Testing Power}
\label{sec:example}
The limiting null distribution have closed-form densities in some special cases, for which the largest validity level $\alpha$ (i.e., the largest type 1 error level so the chi-square test is still valid) can be exactly determined, e.g., $0.05$ in Theorem~\ref{thm4} and at most $0.0875$ in Theorem~\ref{thm6}.
\begin{theorem}
\label{thm4}
Given $m \geq 1$, assume the weights in Theorem~\ref{thm1} satisfy $w_i=\frac{1}{\sqrt{m}}$ for all $i =1,\ldots,m$ and zero otherwise. It follows that \begin{align*}
n \Dcor_{n}(\mathbf{X},\mathbf{Y}) \stackrel{D}{\rightarrow} \frac{\chi^{2}_{m}-m}{\sqrt{m}} \preceq_{0.05} U,
\end{align*}
where $\chi^{2}_{m}$ is the chi-square distribution of degree $m$.
\end{theorem} 

\begin{corollary}
Under the same assumption of Theorem~\ref{thm3}. Let $m_{1} = \left \lfloor{1/w_{1}^{2}}\right \rfloor$ and $m_{2} = \left \lceil{1/w_{1}^{2}}\right \rceil$, it always holds that
\begin{align*}
\frac{\chi^{2}_{m_{2}}-m_{2}}{\sqrt{m_{2}}} \preceq_{\alpha} \sum\limits_{i=1}^{m}w_{i} U_i \preceq_{\alpha} \frac{\chi^{2}_{m_{1}}-m_{1}}{\sqrt{m_{1}}} \preceq_{\alpha}U.
\end{align*}
\end{corollary} 

Next, we offer two random variable examples that achieve the known densities. Theorem~\ref{thm5} shows that the limiting null distribution can equal the centered chi-square distribution. Theorem~\ref{thm6} establishes the normal distribution (which is also used in the t-test in \citet{SzekelyRizzo2013a}) with a simplified condition and a simplified proof, i.e., only one random variable is required to have infinite dimension.
\begin{theorem}
\label{thm5}
When $X$ and $Y$ are two independent binary random variables, $n \Dcor_{n}(\mathbf{X},\mathbf{Y}) \stackrel{D}{\rightarrow} U$.
\end{theorem}

\begin{theorem}
\label{thm6}
Assume $X$ is independent of $Y$, $X$ is continuous, and each dimension of $X$ is exchangeable with positive finite variance. As $n,p \rightarrow \infty$, it holds that
\begin{align*}
&n \Dcor_{n}(\mathbf{X},\mathbf{Y}) \stackrel{D}{\rightarrow} \mc{N}(0,2) \preceq_{0.0875} U.
\end{align*}
\end{theorem}

As the distance correlation t-test essentially uses a t-transformation of $\mc{N}(0,2)$, we can view the t-test and the chi-square test as two-sides of the permutation test. In particular, the distance correlation t-test is an invalid test that slightly inflates the type $1$ error level, while the chi-square test is valid but always conservative in power:
\begin{corollary}
\label{cor0}
At any type 1 error level $\alpha \leq 0.05$, denote the testing power of distance correlation chi-square test, distance correlation t-test, and the permutation test as $\beta_{\alpha}^{\chi}, \beta_{\alpha}^{t}, \beta_{\alpha}$ respectively. At any $n$ and $\alpha$ such that Theorem~\ref{cor1} holds, there exists $\alpha_{1} \in (0,\alpha]$ and $\alpha_{2} \in (\alpha,0.0875]$ such that
\begin{align*}
\beta_{\alpha}^{\chi} = \beta_{\alpha_{1}} \leq \beta_{\alpha} < \beta_{\alpha_{2}} = \beta_{\alpha}^{t}.
\end{align*}
The actual $\alpha_{1}$ and $\alpha_{2}$ depend on the metric choice and marginal distributions. 
\end{corollary}

Finally, there is no other valid test of distance correlation that is as fast and as powerful as the chi-square test:
\begin{corollary}
\label{cor2}
At any $n$ and $\alpha$ such that Theorem~\ref{cor1} holds, the chi-square test is the most powerful test among all valid tests of distance correlation using known distributions. Namely, for any valid test $z$ of distance correlation using a fixed distribution, it always holds that $\beta_{\alpha}^{\chi} \geq \beta_{\alpha}^{z}$.
\end{corollary}

\section{Numerical Experiments}
\label{sec:simu}

This section evaluates the numerical advantages of the chi-square test. First, we show that the centered chi-square distribution approximates the true null distribution of bias-corrected distance correlation in simulation. Second, the chi-square test is compared to existing tests on 1D data. Then the comparison is carried out on multivariate and high-dimensional data. And finally, we use a real data application to demonstrate the practical usage and superior performance of the chi-square test. Simulation function details are in Appendix~\ref{app_sim}.

\subsection{Null Distribution Approximation}
\label{simu0}

The top row of Figure~\ref{fig1} visualizes the centered chi-square distribution $U$ and the normal distribution $N(0,2)$ (divide each by $n$), and compare them with the actual null distribution of $\Dcor_{n}(\mathbf{X},\mathbf{Y})$ in varying dimensions. The centered chi-square distribution and the normal distribution are plotted by a solid line and an dashed line, respectively. We set sample size at $n=100$, generate independent $\mathbf{X}$ and $\mathbf{Y}$ for $r=10$,$000$ replicates and different $p,q$, and plot the null distribution in dotted line. The left panel shows the null distributions at $p=q=1$, the center panel is for $p=q=10$, and the right panel is for $p=q=100$. As expected from the theorems, the upper tail of the null distribution for $\alpha \leq 0.05$ (equivalently y-axis greater than $0.95$) always lies between and gradually shifts from $U$ to $N(0,2)$ as dimension increases.

The bottom row of Figure~\ref{fig1} shows the weights used in the corresponding limiting null distribution, associated with the ordered eigenvalue list $\{\lambda_i, i=1,\ldots,10\}$ normalized by $-(\sum\limits_{i=1}^{n}\lambda_i^2)^{0.5}$ (see proof of Theorem~\ref{thm1}). When $p=1$, the leading weight plays a dominating role such that $n \Dcor_{n}(\mathbf{X},\mathbf{Y}) \stackrel{D}{\approx} U$. As $p,q$ increase, all weights becoming similar such that $n \Dcor_{n}(\mathbf{X},\mathbf{Y}) \stackrel{D}{\rightarrow} N(0,2)$. Therefore, the most conservative the chi-square test can be is to approximate a normal distribution by a chi-square distribution of same mean and variance.

\begin{figure}
\includegraphics[width=1.0\textwidth,trim={1cm 0.5cm 1cm 0.5cm},clip]{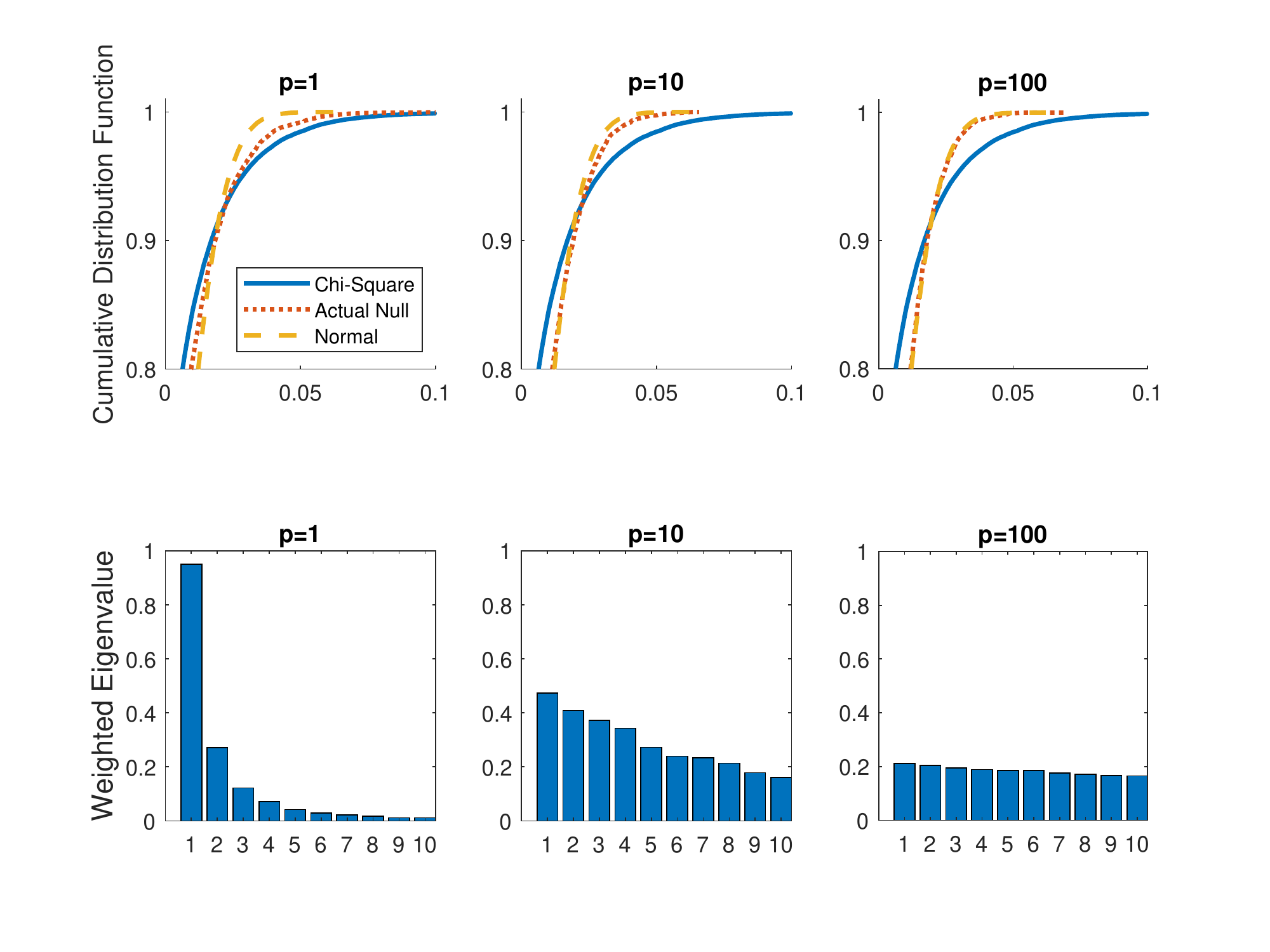}
\caption{The top row compares the centered chi-square distribution, the normal distribution, and the actual null distribution of distance correlation in case of varying dimensions. The bottom row shows the weights used in the limiting null distribution in each case.} 
\label{fig1}
\end{figure} 

\subsection{One-Dimensional Random Variables}
\label{simu1}

Here we evaluate distance correlation chi-square test, permutation test, distance correlation t-test, HSIC Gamma test (only applicable to biased HSIC), and the subsampling method for four different 1D simulations: linear, quadratic, spiral, and independent. We extensively evaluated the tests on many different dependency types and decided to illustrate four representative simulations only, as the phenomenon is qualitatively similar throughout. All simulations are one-dimensional in this section, that is, $p=q=1$. In each simulation, we sample $n=20,40,\ldots, 200$ points, generate sample data $1$,$000$ times, run each test and reject at $\alpha=0.05$ level, and compute how often the test is correctly rejected (the testing power), which is then plotted against the sample size.

The top row of Figure~\ref{fig2} shows the power of distance correlation under Euclidean distance, the middle row shows the power of Hilbert-Schmidt independence criterion (equivalently the distance correlation using Gaussian kernel), and the bottom row shows the running time for the top row in log scale. The performance is the same throughout all dependency types and sample size and metric choice: the permutation test is the benchmark for testing power, but significantly slower than the fast alternatives; the distance correlation chi-square test has almost the same testing power as the benchmark permutation test, and the fastest method; the distance correlation t-test is equally fast, but consistently inflates the power in most cases, e.g., it has a power of $0.07$ for independence vs about $0.05$ of the permutation test; the subsampling method always yields degraded power. 

Note that the HSIC Gamma test requires some parameter estimation using mean, variance, and kernel bandwidth, thus only applicable to the middle row for the biased HSIC. Its testing power is almost the same as the chi-square and permutation tests. This does not mean the chi-square test approximates the Gamma test in distribution. Rather, it implies that the Gamma test approximates the null distribution of biased HSIC well; the chi-square test approximates the bias-corrected null distribution well; then the permutation test has almost the same testing power with or without bias-correction.

\begin{figure}
\includegraphics[width=1.0\textwidth,trim={2cm 0cm 2cm 0cm},clip]{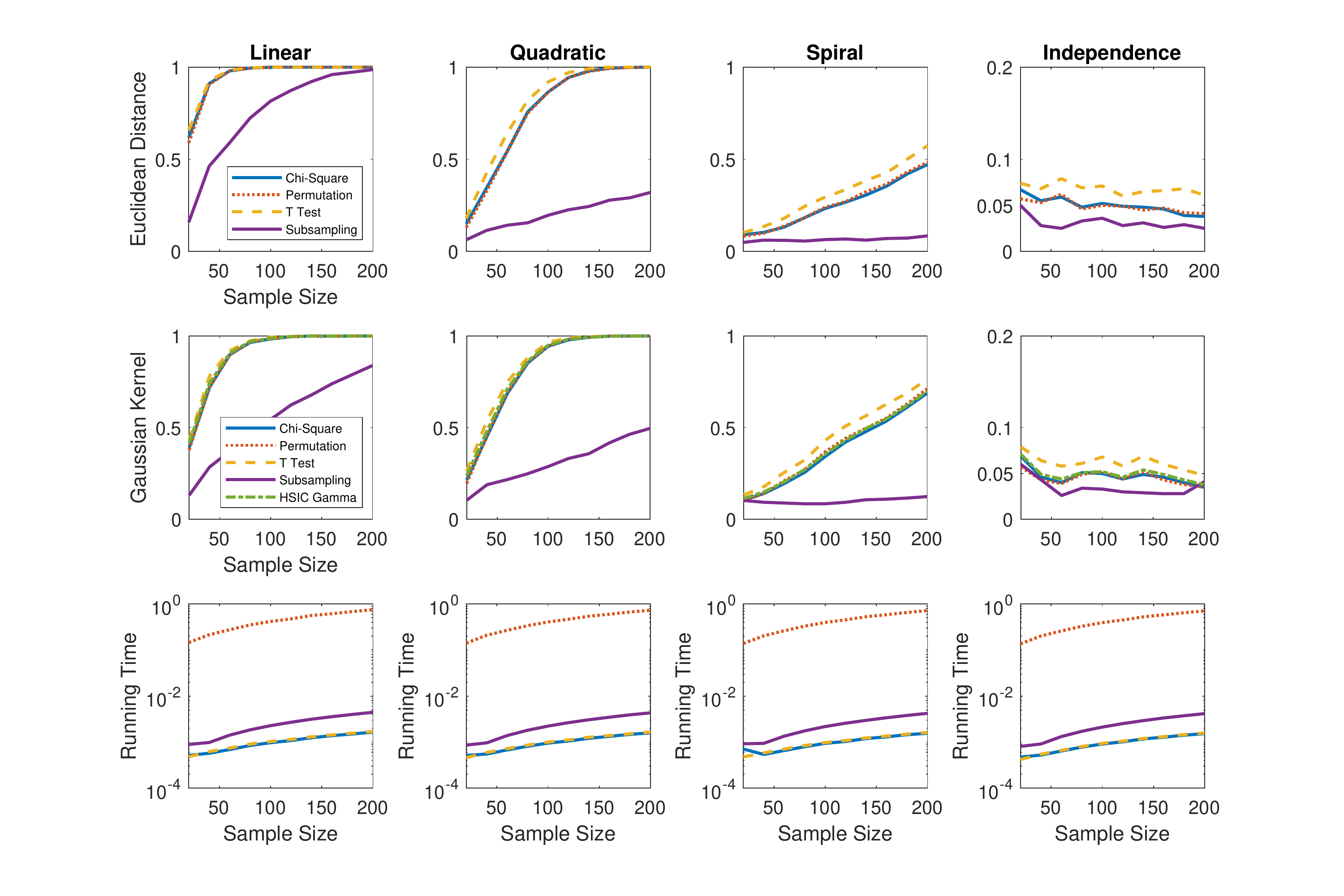}
\caption{Evaluate distance correlation using different tests for linear, quadratic, spiral, and independent simulations. The top row shows the power using the Euclidean distance, the center row shows the power using the Gaussian kernel, and the bottom row shows the running time (in log scale) for each method in the top row.}
\label{fig2}
\end{figure} 

\subsection{Increasing-Dimensional Random Variables}
\label{simu2}
We consider four multivariate settings: equal variance (each dimension is exchangeable with same variance), minimal variance (first few dimensions has same variance while remaining dimensions have very small variance), dependent coordinates (consecutive dimensions are dependent), and varying marginals (the marginal distribution of each dimension is different). We fix the sample size and $q=1$, increase $p$ accordingly in each simulation, and compute the testing power at $\alpha=0.05$ based on $1000$ Monte-Carlo replicates.

The testing power and the running time are plotted against dimension in Figure~\ref{fig3}, offering almost the same interpretation as Figure~\ref{fig2}. In particular, the equal variance simulation is the only setting here satisfying the assumption of Theorem~\ref{thm6}, in which case the t-test only minimally inflate the testing power and the chi-square test exhibits a slightly more conservative testing power vs the permutation test. In the other three high-dimensional settings, the dimensions are no longer exchangeable, and the chi-square test has almost the same power as the permutation test. In terms of running time, the chi-square test and the t-test are the best, which do not increase much as dimension increases. 

Finally, we evaluate the testing power for the multivariate simulation in Figure~\ref{fig1}, and present the results in Figure~\ref{fig4}. Regardless of $p,q,n$, the chi-square test is always similar in power as the permutation test and slightly conservative.

\begin{figure}
\centering
\includegraphics[width=1.0\textwidth,trim={2cm 0cm 2cm 0cm},clip]{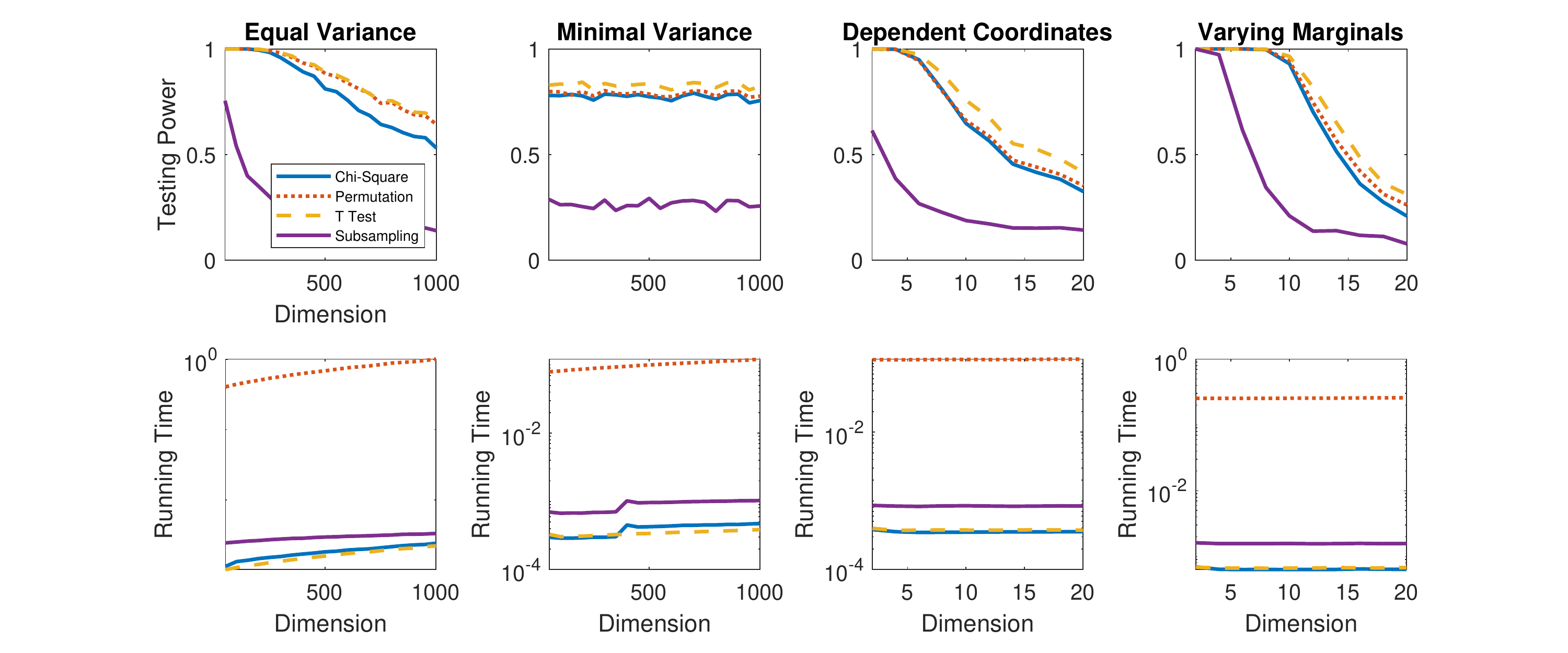}
\caption{Evaluate distance correlation using different tests for four increasing-dimensional simulations using Euclidean distance. The first row shows the testing power in each simulation, and the second row shows the running time in log scale for each method in the respective first row.}
\label{fig3}
\end{figure}

\begin{figure}
\centering
\includegraphics[width=0.7\textwidth,trim={1cm 0.5cm 1cm 0.5cm},clip]{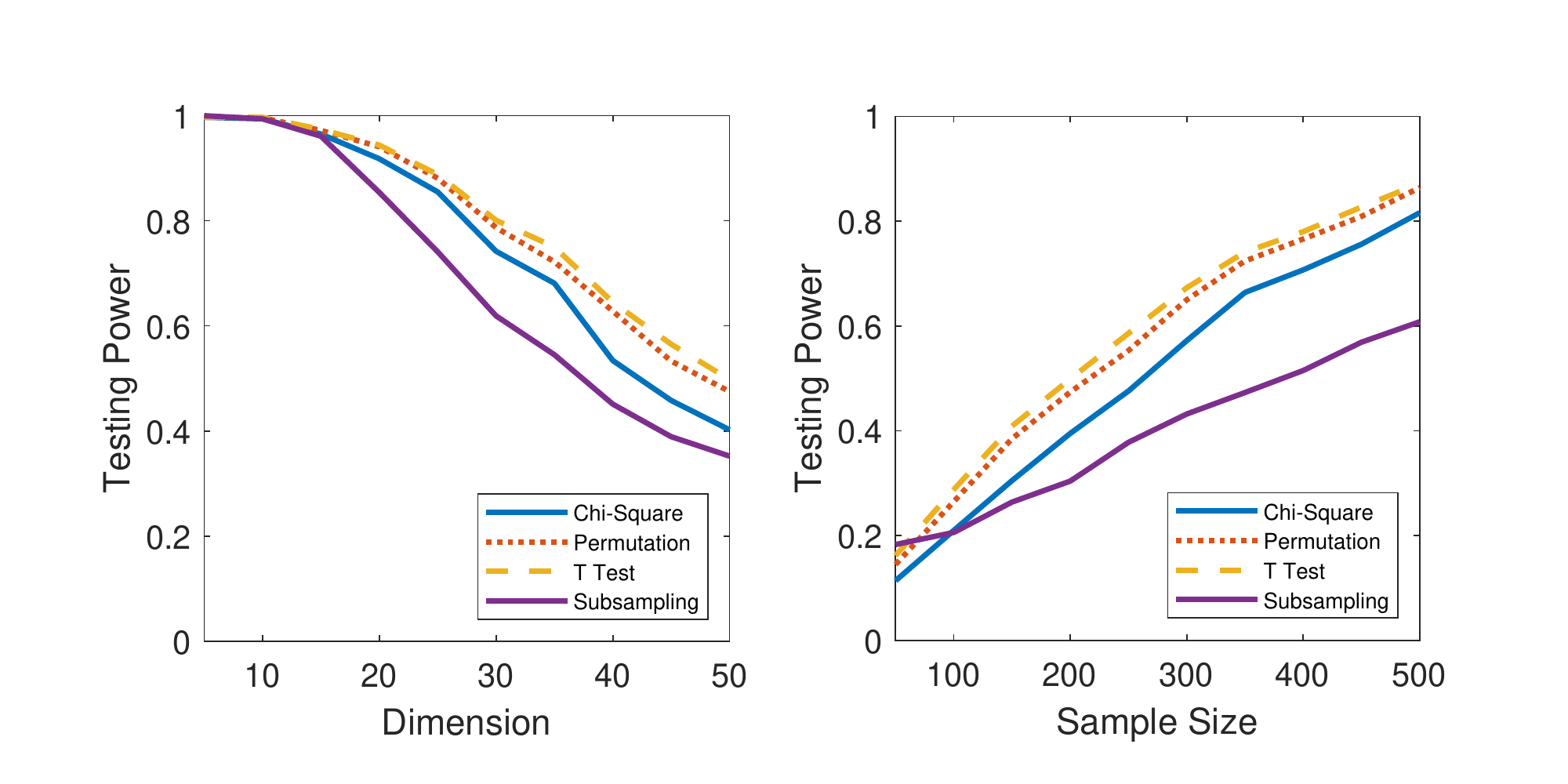}
\caption{The Testing Power for the simulation in Figure~\ref{fig1}. The left panel fix $n=200$, and let $p,q$ increases; the right panel set $p=q=50$, and let $n$ increases. } 
\label{fig4}
\end{figure} 

\subsection{Real Data Application}

Here we apply the tests to do feature selection on a proteomics data \citep{mgc1,PMID21248225}. The data consist of $318$ proteolytic peptides measurements, derived from the blood samples of $98$ individuals harboring pancreatic ($n=10$), ovarian ($n=24$), colorectal cancer ($n=28$), and healthy controls  ($n=33$). We would like to identify potential biomarkers for the pancreatic cancer, because it is lethal and no clinically useful biomarkers are currently available \citep{Bhat2012}.

For each biomarker, we first apply the distance correlation permutation test (using $500$ replicates) between the blood measurement of each peptide and the class label ($1$ for pancreatic cancer, $0$ otherwise). This yields $318$ p-values, and we take all peptides with p-value less than $0.05$ as the positive ones (in total 19 of them), and the remaining 299 peptides as the negative ones. Then we apply the chi-square test to each peptide vs the class label, and compute the true positive (i.e., when a peptide has p-value $<0.05$ from the chi-square test, whether the peptide also has p-value $<0.05$ from the permutation test) and true negative. Repeat for the t-test and subsampling method. 

The result is summarized in Table~\ref{table1}. The chi-square test is fast and almost perfect for both true positive and true negative. The t-test is as fast, but inflates one of two errors. The subsampling is actually non-performing (all peptides have p-value $>0.1$ so every peptide is a negative) --- this is because due to the imbalanced label vector, most subsampled labels are a constant $0$, so the subsampled statistics end up being $0$ most of the time.

\begin{table*}[!ht]
\centering
\caption{Report the running time (in seconds), true positive and true negative for each method. Permutation test is the benchmark for determining the positive and negative peptides. Among all the approaches, the chi-square test is the fastest with almost perfect result. }
\label{table1}%
\begin{tabular}{l|l l l }
\toprule
Method & Running Time & True Positive & True Negative \\
\midrule
Permutation Test 	& 149.8  & 100\% & 100\% \\
\midrule
Chi-Square Test 	& \textbf{0.32}  & \textbf{92\%} & \textbf{98\%} \\
T Test   & 0.34 & \textbf{92\%} & 68\% \\
Subsampling & 0.48 & NaN & 94\% \\
\bottomrule
\end{tabular}
\end{table*}

\section{Discussion}
\label{sec:dis}
This paper proposes a new chi-square test for testing independence. It is very computationally efficient and scalable to big data, valid and universally consistent, applicable to general metric and kernel choices and other tasks like K-sample and partial, achieves similar power as the permutation test, and compares very favorably against all existing tests in simulations and real data. We expect the chi-square test to be widely used in practice due to its computational and performance advantages, and we plan to further investigate the stochastic dominance in theory and its potential applications in other inference tasks.

In particular, by Corollary~\ref{cor0} the chi-square test is always more conservative than the permutation test. But the power loss is negligible in 1D data and minimal in high-dimensional data as shown in Figure~\ref{fig2} and Figure~\ref{fig3}. In comparison, the t-test always inflates the testing power. As shown in Figure~\ref{fig1}, at the type 1 error level $0.05$ the t-test inflates the error to $(0.05,0.07]$, while the chi-square test has a conservative error within $[0.02,0.05]$. Viewed in another way, the t-test at $\alpha=0.05$ is equivalent in power to a permutation test at $\alpha \in (0.05,0.07]$, while the chi-square test at $\alpha=0.05$ is equivalent in power to a permutation test at $\alpha \in [0.02,0.05]$. Inflating the type 1 error slightly may not be a big issue in a single test, but can cause larger deviation in multiple comparison, as evidenced tin the real data application; whereas the chi-square test is not susceptible to this issue.

\section*{Acknowledgment}
This work was supported by the National Science Foundation award DMS-1921310, the National Institute of Health award R01MH120482, and DARPA L2M program FA8650-18-2-7834.
The authors thank Dr. Austin Brockmeier, Mr. Ronak Mehta, Dr. Yuexiao Dong, Dr. Carey Priebe for helpful discussions and suggestions. We also thank the editors and reviewers for the valuable suggestions that greatly improved the paper in organization, expositions, and experiments.

\bibliographystyle{apalike}
\bibliography{mgc}

\clearpage
\appendix

\setcounter{figure}{0}
\setcounter{algorithm}{0}
\renewcommand{\thealgorithm}{A\arabic{algorithm}}
\renewcommand{\thefigure}{E\arabic{figure}}
\renewcommand{\thesubsection}{\thesection.\arabic{subsection}}
\renewcommand{\thesubsubsection}{\thesubsection.\arabic{subsubsection}}
\pagenumbering{arabic}
\renewcommand{\thepage}{\arabic{page}}

\bigskip
\begin{center}
{\large\bf APPENDIX}
\end{center}

\section{Background} 
\label{sec:bg}

This section contains necessary background for the proofs.

\subsection{Biased and Bias-corrected Sample Distance Correlation}

Let the paired sample data, which is assumed independently and identically distributed as $F_{XY}$, be denoted by 
\begin{align*}
(\mathbf{X},\mathbf{Y}) = \{(x_{i},y_{i}) \in \mathbb{R}^{p+q}, i=1,\ldots,n\}.
\end{align*}
Given a distance metric $d(\cdot,\cdot)$ such as the Euclidean metric, let $\mathbf{D}^{\mathbf{X}}$ denote the $n \times n$ distance matrix of $\mathbf{X}$ with $\mathbf{D}^{\mathbf{X}}_{ij}=d(x_i,x_j)$, $\mathbf{D}^{\mathbf{Y}}$ denote the distance matrix of $\mathbf{Y}$, and $\mathbf{H}=\mathbf{I}-\frac{1}{n}\mathbf{J}$ denote the $n \times n$ centering matrix where $\mathbf{I}$ is the identity matrix and $\mathbf{J}$ is the matrix of ones. The biased sample distance correlation was proposed in \citet{SzekelyRizzoBakirov2007} with an elegant matrix formulation:
\begin{align*}
&\Dcov_{n}^{b}(\mathbf{X},\mathbf{Y})= \frac{1}{n^2}\trace{\mathbf{H}\mathbf{D}^{\mathbf{X}}\mathbf{H}\mathbf{H}\mathbf{D}^{\mathbf{Y}}\mathbf{H}}, \\
&\Dcor_{n}^{b}(\mathbf{X},\mathbf{Y})= \frac{\Dcov_{n}^{b}(\mathbf{X},\mathbf{Y})}{\sqrt{\Dcov_{n}^{b}(\mathbf{X},\mathbf{X})\Dcov_{n}^{b}(\mathbf{Y},\mathbf{Y})}} \in [0,1],
\end{align*}
where $\Dcov_{n}^{b}$ denotes the biased sample distance covariance and $\Dcor_{n}^{b}$ denotes the biased sample distance correlation. The bias-corrected version was later introduced via the following bias correction \citep{SzekelyRizzo2014}: compute a modified matrix $\mathbf{C}^{\mathbf{X}}$ as
\begin{align*}
\mathbf{C}^{\mathbf{X}}_{ij}=
 \begin{cases}
 \mathbf{D}^{\mathbf{X}}_{ij}-\frac{1}{n-2}\sum\limits_{t=1}^{n} \mathbf{D}^{\mathbf{X}}_{it}-\frac{1}{n-2}\sum\limits_{s=1}^{n} \mathbf{D}^{\mathbf{X}}_{sj}+\frac{1}{(n-1)(n-2)}\sum\limits_{s,t=1}^{n}\mathbf{D}^{\mathbf{X}}_{st}, \ i \neq j \\
 0, \mbox{ otherwise},
 \end{cases}
\end{align*}
and similarly compute $\mathbf{C}^{\mathbf{Y}}$ from $\mathbf{D}^{\mathbf{Y}}$. The bias-corrected sample distance covariance and correlation are
\begin{align*}
& \Dcov_{n}(\mathbf{X}, \mathbf{Y}) = \frac{1}{n(n-3)}\trace{\mathbf{C}^{\mathbf{X}}\mathbf{C}^{\mathbf{Y}}},\\
&\Dcor_{n}(\mathbf{X},\mathbf{Y})= \frac{\Dcov_{n}(\mathbf{X},\mathbf{Y})}{\sqrt{\Dcov_{n}(\mathbf{X},\mathbf{X})\Dcov_{n}(\mathbf{Y},\mathbf{Y})}} \in [-1,1].
\end{align*}
Namely, $\mathbf{C}^{\mathbf{X}}$ always sets the diagonals to $0$ and slightly modifies the off-diagonal entries from $\mathbf{H}\mathbf{D}^{\mathbf{X}}\mathbf{H}$. If $n<4$ or the denominator term is not a positive real number, the bias-corrected sample distance correlation is set to $0$. 

As long as the metric $d(\cdot,\cdot)$ is of strong negative type such as the Euclidean metric \citep{Lyons2013}, distance correlation satisfies the following:
\begin{align*}
\Dcor_{n}(\mathbf{X},\mathbf{Y}) \stackrel{n \rightarrow \infty}{\rightarrow} 0 \mbox{ if and only if independence, }
\end{align*}
which guarantees a universally consistent statistic for testing independence. Moreover, it is unbiased in the following sense:
\begin{align*}
\EE(\Dcor_{n}(\mathbf{X},\mathbf{Y})) = 0 \mbox{ when $X$ is independent of $Y$},
\end{align*}
which is not satisfied by the biased version.

Instead of using a strong negative type distance metric, one may use a characteristic kernel for $d(\cdot,\cdot)$, i.e., $\mathbf{D}^{\mathbf{X}}$ and $\mathbf{D}^{\mathbf{Y}}$ become two kernel matrices, and all above formulation still hold. When one uses the Gaussian kernel with median distance as the bandwidth, the resulting correlation becomes the Hilbert-Schmidt independence criterion. As the theorems hold for any strong negative type metric or any characteristic kernel, we shall consistently use the distance correlation naming regardless of whether a metric or kernel is used for $d(\cdot,\cdot)$. 

\subsection{Null Distribution of Distance Correlation}

The goal for fast testing is to approximate the null distribution of distance correlation via a known distribution. Then given any sample test statistic, one can immediately compute the p-value, and reject the independence hypothesis when it is smaller than a pre-specified critical level $\alpha$. 

From \citet{zhang2018}, the limiting null distribution of unbiased distance covariance satisfies
\begin{align} \label{eq1}
n \Dcov_{n}(\mathbf{X},\mathbf{Y}) \stackrel{D}{\rightarrow} \sum\limits_{i,j=1}^{\infty} \lambda_{i} \mu_{j} (\mc{N}_{ij}^{2}-1),
\end{align}
where $\{\lambda_{i}\}$ are the limiting eigenvalues of $\mathbf{H}\mathbf{D}^{X}\mathbf{H}/n$, $\{\mu_{j}\}$ are the limiting eigenvalues of $\mathbf{H}\mathbf{D}^{Y}\mathbf{H}/n$, $\mc{N}_{ij}$ are identically and independently distributed standard normal random variables, and the summation index sums over $i=1,\ldots,\infty$ and $j=1,\ldots,\infty$. The limiting null distribution using sample eigenvalues is shown to almost equal the finite-sample null distribution for $n \geq 20$ (see \citet{Lyons2013} and \citet{zhang2018} for more details), so it suffices to use the limiting null. The eigenvalues $\{\lambda_{i}\}$ and $\{\mu_{j}\}$ can vary significantly for different metric or kernel choices $d(\cdot,\cdot)$ and different marginal distributions $F_{X}$ and $F_{Y}$. Therefore, no fixed distribution can perfectly approximate the null all the time. 

In a brute-force manner, the sample eigenvalues can be estimated via eigen-decomposition of the sample matrices, then the null distribution can be simulated by generating $n^2$ independent normal distributions. This method has the best testing power (i.e., almost the same as permutation test for $n \geq 20$) but requires $\mc{O}(n^3)$ time complexity thus too costly. Alternatively, one may compute subsampled correlations and take average, which follows a normal distribution via the central limit theorem. However, it is well-known and provable that a subsampled statistic yields an inferior testing power, because the estimated null distribution is a very conservative one with enlarged variance. These two approaches are summarized in \citet{zhang2018}. 

The standard permutation test works as follows: for each replicate, permute the observations in $\mathbf{Y}$ (row indices of the matrix) by a random permutation $\pi$, denote the permuted sample data as $\mathbf{Y}^{\pi}$, and compute the permuted statistic $\Dcor(\mathbf{X},\mathbf{Y}^{\pi})$. Repeat for $r$ such random permutations, and compute a set of permuted statistics $\{\Dcor(\mathbf{X},\mathbf{Y}^{\pi})\}$. Then the p-value is the fraction of times the observed test statistic is more extreme than the permuted test statistics. This is summarized in Algorithm~\ref{alg2}. The random permutation effectively breaks dependencies within the sample data and well-approximates the actual null distribution. The permutation test is the default approach in almost every independence testing methodology, and provably a valid and consistent test with any consistent dependence measure \citep{mgc2}, not just distance correlation. Also note that distance correlation and distance covariance share the same p-value under permutation test, because the covariance to correlation transformation is invariant to permutation.

\begin{algorithm}
\caption{Permutation Test for Independence}
\label{alg2}
\begin{algorithmic}
\Require Paired sample data $(\mathbf{X},\mathbf{Y})=\{(x_i,y_i) \in \mathbb{R}^{p+q} \mbox{ for }  i \in [n]\}$, and the number of random permutation $r$.
\Ensure The test statistic $C$ and the p-value $p$.
\Function{PermutationTest}{$\mathbf{X},\mathbf{Y}, r$}
\State $C=\textsc{Stat}(\mathbf{X},\mathbf{Y})$; 
\Comment{can be any dependency measure not just distance correlation}

\For{$s=1,\ldots,r$}
\State $\pi=\texttt{randperm}(n)$;
\Comment{generate a random permute index}
\State  $cp(s)=\textsc{Stat}(\mathbf{X}(\pi),\mathbf{Y})$;
\Comment{$cp$ stores the permuted statistics}
\EndFor

\State $p=\sum_{s=1}^{r} \mb{I}(cp(s)>C)/r$
\Comment{the percentage the permuted statistics is larger}
\EndFunction
\end{algorithmic}
\end{algorithm}

A popular test using a known distribution is the distance correlation t-test \citep{SzekelyRizzo2013a}, which approximates the null distribution by a normal distribution of mean $0$ and variance $2$. When $X$ and $Y$ are independent, assume each dimension of $X$ and $Y$ are independently and independently distributed (or exchangeable) with positive finite variance, it was proved that
\begin{align}
\label{eq2}
\sqrt{n^2-3n-2} \cdot \Dcor_{n}(\mathbf{X},\mathbf{Y}) \stackrel{D}{\rightarrow} \mc{N} (0,2)
\end{align}
as $n,p,q \rightarrow \infty$. The t-distribution transformation and the corresponding t-test follow from the normal distribution. Therefore, the distance correlation t-test has a constant time complexity and enjoys a similar testing power as the permutation test under required condition. However, there has been no investigation on its testing performance out of the high-dimension assumption.

\section{Proofs}
\subsection{Proof of Theorem~\ref{thm7}}
\begin{proof}
From Theorem~\ref{cor1}, $U$ dominates $n\Dcor(\mathbf{X},\mathbf{Y})$ in upper tail which the actual null converges to. Therefore, there exists $n'$ and $\alpha^{'}$ such that the test correctly controls the type 1 error level for any $\alpha \leq \alpha'$ and sample size $n \geq n'$. For example, when $\alpha^{'} = 0.05$ from Theorem~\ref{thm4}, the test is expected to be valid at any type 1 error level no more than $0.05$ at sufficiently large $n$.

For consistency: at any $\alpha<2\Phi(1)-1$, $F_{U}^{-1}(1-\alpha)$ is a positive and fixed constant. When $X$ is dependent of $Y$, $\Dcor_{n}(\mathbf{X},\mathbf{Y})$ converges to a non-zero positive constant, such that $n\Dcor_{n}(\mathbf{X},\mathbf{Y}) \rightarrow +\infty > F_{U}^{-1}(1-\alpha)$ and the test is always correctly rejected asymptotically. 

Therefore, the distance correlation chi-square test is valid and universally consistent for testing independence.
\end{proof}

\subsection{Proof of Theorem~\ref{thm8}}
As the p-value computation in Algorithm~\ref{alg1} takes $\mc{O}(1)$, it suffices to show the bias-corrected distance correlation can be computed in $\mc{O}(n \log n)$ for one-dimensional data using Euclidean distance. Denote the distances and centered distances as
\begin{align*}
A_{ij}  =\|x_i-x_j\|_{2}, \ \ \ \  &B_{ij} =\|y_i-y_j\|_{2}\\
A_{i \cdot}  =\sum\limits_{j=1}^{n} A_{ij}, \ \ \ \ &B_{i \cdot} =\sum\limits_{j=1}^{n} B_{ij}\\
A_{\cdot \cdot}  =\sum\limits_{i=1}^{n} A_{i \cdot}, \ \ \ \  &B_{\cdot \cdot} =\sum\limits_{i=1}^{n} B_{i \cdot}.
\end{align*}
and define
\begin{align*}
T_1 =\sum\limits_{i,j=1}^{n} A_{ij}B_{ij}, \ \ T_2 =\sum\limits_{i=1}^{n}A_{i \cdot}B_{i\cdot}, \ \ T_3 &=A_{\cdot \cdot} B_{\cdot \cdot}.
\end{align*}
It was shown in \citet{Hu2018} that $T_1, T_2, T_3$ can be computed in $\mc{O}(n \log n)$ for one-dimensional data using Euclidean metric. Therefore, it suffices to prove the following lemma:

\begin{lemma}
The unbiased distance covariance can be expressed into
\begin{align*}
\Dcov_{n}(\mathbf{X},\mathbf{Y})=&\frac{T_1}{n(n-3)} - \frac{2T_2}{n(n-2)(n-3)}+\frac{T_3}{n(n-1)(n-2)(n-3)}.
\end{align*}
Then unbiased distance covariance and correlation can be computed in $\mc{O}(n \log n)$ for one-dimensional data using Euclidean distance.
\end{lemma}

\begin{proof}
The unbiased distance covariance can be decomposed into
\begin{align*}
\Dcov_{n}(\mathbf{X},\mathbf{Y})=& \frac{1}{n(n-3)}\sum\limits_{i \neq j}^{n}(A_{ij} -\frac{1}{n-2}A_{i \cdot}-\frac{1}{n-2}A_{j \cdot}+\frac{1}{(n-1)(n-2)}A_{\cdot  \cdot})\\
&\cdot (B_{ij} -\frac{1}{n-2}B_{i \cdot}-\frac{1}{n-2}B_{j \cdot}+\frac{1}{(n-1)(n-2)}B_{\cdot  \cdot}) \\
=& \frac{T_1}{n(n-3)} - \frac{2T_2}{n(n-2)(n-3)}+\frac{T_3}{n(n-1)(n-2)(n-3)} \\
&-\frac{T_2}{n(n-2)(n-3)}+\frac{(n-1)T_2}{n(n-2)^2(n-3)}+\frac{T_3-T_2-T_3}{n(n-2)^2(n-3)}\\
&-\frac{T_2}{n(n-2)(n-3)}+\frac{(n-1)T_2}{n(n-2)^2(n-3)}+\frac{T_3-T_2-T_3}{n(n-2)^2(n-3)}\\
&+\frac{T_3}{n(n-1)(n-2)(n-3)}-\frac{2T_3}{n(n-2)^2(n-3)}+\frac{T_3}{(n-1)(n-2)^2(n-3)}\\
=& \frac{T_1}{n(n-3)} - \frac{2T_2}{n(n-2)(n-3)}+\frac{T_3}{n(n-1)(n-2)(n-3)}.
\end{align*}
To compute the bias-corrected distance correlation, one needs to compute $\Dcov_{n}(\mathbf{X},\mathbf{Y})$, $\Dcov_{n}(\mathbf{X},\mathbf{X})$, and $\Dcov_{n}(\mathbf{Y},\mathbf{Y})$, all of which now take $\mc{O}(n \log n)$. Therefore, the bias-corrected distance correlation can be computed in $\mc{O}(n \log n)$. 
\end{proof}

\subsection{Proof of Corollary~\ref{cor3} and Corollary~\ref{cor4}}
\begin{proof}
From \citet{exact1}, a valid and consistent independence test is also valid and consistent for K-sample testing. Thus Corollary~\ref{cor3} follows immediately from Theorem~\ref{thm7}.

From \citet{SzekelyRizzo2014}, the partial distance correlation is also trace of the product of two modified distance matrix. Therefore $\sct{PDcor}(X,Y;Z)$ has the same limiting null distribution as in Equation~\ref{eq1}. Thus Corollary~\ref{cor4} also follows from Theorem~\ref{thm7}.
\end{proof}

\subsection{Proof of Theorem~\ref{thm1}}
\begin{proof}
Recall from Equation~\ref{eq1} that the limiting null distribution of distance covariance satisfies
\begin{align*}
&n \Dcov_{n}(\mathbf{X},\mathbf{Y}) \stackrel{D}{\rightarrow}  \sum\limits_{i,j=1}^{\infty} \lambda_{i} \mu_{j} (\mc{N}_{ij}^{2}-1),
\end{align*}
where $\{\lambda_{i}\}$ are the limiting eigenvalues of $\mathbf{H}\mathbf{D}^{\mathbf{X}}\mathbf{H}/n$ and $\{\mu_{j}\}$ are the limiting eigenvalues of $\mathbf{H}\mathbf{D}^{\mathbf{Y}}\mathbf{H}/n$. Then the distance variance always satisfies
\begin{align*}
&\lim \limits_{n\rightarrow \infty}(\Dcov_{n}(\mathbf{X}, \mathbf{X}) - \sum\limits_{i=1}^{n} \lambda_{i}^2) \\
= &\lim \limits_{n\rightarrow \infty}(\Dcov_{n}(\mathbf{X}, \mathbf{X}) - \Dcov_{n}^{b}(\mathbf{X}, \mathbf{X})) \\
\rightarrow & \ 0,
\end{align*}
where the third line follows from the fact that both unbiased and biased statistics converge to a same constant, and the second line follows from 
\begin{align*}
\Dcov_{n}^{b}(\mathbf{X}, \mathbf{X}) = \frac{1}{n^2}\trace{\mathbf{H}\mathbf{D}^{\mathbf{X}}\mathbf{H}\mathbf{H}\mathbf{D}^{\mathbf{X}}\mathbf{H}} = \sum\limits_{i=1}^{n} \lambda_{i}^2.
\end{align*}
Therefore,
\begin{align*}
&\Dcov_{n}(\mathbf{X},\mathbf{X}) \rightarrow \sum\limits_{i=1}^{\infty} \lambda_{i}^2,\\
&\Dcov_{n}(\mathbf{Y},\mathbf{Y}) \rightarrow \sum\limits_{j=1}^{\infty} \mu_{j}^2 ,\\
& n\Dcor_{n}(\mathbf{X},\mathbf{Y}) \stackrel{D}{\rightarrow} \sum\limits_{i,j=1}^{\infty} w_{ij} (\mc{N}_{ij}^{2}-1).
\end{align*}
where $w_{ij}= \frac{\lambda_{i} \mu_{j}}{ \sqrt{\sum\limits_{i=1}^{\infty} \lambda_{i}^2 \sum\limits_{j=1}^{\infty} \mu_{j}^2}}$.

A strong negative type metric is always of negative type, and a characteristic kernel is always a positive definite kernel. When the distance metric is of negative type, the two matrices are negative definite and all eigenvalues are all non-positive. When positive definite kernels are used, then these eigenvalues are all non-negative. In either case, the product $\{\lambda_{i} \mu_{j}\}$ is always non-negative such that $w_{ij} \in [0,1]$. Moreover, for any $n \geq 1$ it always holds that
\begin{align*}
\sum\limits_{i,j=1}^{n} w_{ij}^2 &= \frac{\sum\limits_{i,j=1}^{n} \lambda_{i}^{2} \mu_{j}^{2}}{ \sum\limits_{i=1}^{n} \lambda_{i}^2 \sum\limits_{j=1}^{n} \mu_{j}^2} =1.
\end{align*}

Note that in the special case that either $X$ or $Y$ is constant, all eigenvalues are $0$ so the correlation equals $0$ instead. This corresponds to a trivial independence case, and all dominance / validity / consistency results hold trivially. 
\end{proof}

\subsection{Proof of Theorem~\ref{thm3}}
\begin{proof}
We prove this theorem by induction. Lemma~\ref{lem1} proves the initial step at $m=2$ by analyzing the smaller weight $w_{2} \in (0,\frac{1}{\sqrt{2}}]$, and Lemma~\ref{lem2} proves the induction step by considering the smallest weight $w_{m} \in (0,\frac{1}{\sqrt{m}}]$. As the weights are square summed to $1$, the smallest positive weight must satisfy $w_{m} \in (0,\frac{1}{\sqrt{m}}]$, and the largest weight must satisfy $w_{1} \in [\frac{1}{\sqrt{m}},1]$ . It suffices to consider all positive weights, because it reduces to the $m-1$ case when $w_{m}=0$. Moreover, if $m=1$, the summation density equals $U$ which satisfies $f_{U}(x)=O(e^{-x/2})$; and if $w_i=\frac{1}{\sqrt{m}}$ for all $i$, it corresponds to Theorem~\ref{thm4}. 
\end{proof}

\begin{lemma}
\label{lem1}
Suppose both $U$ and $V$ are independently and identically distributed as the centered chi-square distribution, and $w \in (0,1/\sqrt{2}]$. The density of $\sqrt{1-w^2} U+w V$ decays exponentially at the rate $O(e^{-x/(2\sqrt{1-w^2})})$. 
\end{lemma}

\begin{proof}
The summation density equals
\begin{align*}
 & f_{w U + \sqrt{1-w^2} V} (x) = \int_{-w}^{+\infty}f_{\sqrt{1-w^2} V}(x-z)f_{wU}(z)dz \\
 =& \frac{b^2}{w\sqrt{1-w^2}} \int_{-w}^{x+\sqrt{1-w^2}} ((\frac{x-z}{\sqrt{1-w^2}}+1)(\frac{z}{w}+1))^{-0.5}e^{-(\frac{x-z}{\sqrt{1-w^2}}+\frac{z}{w}+2)/2} dz \\
 =& e^{-\frac{x}{2\sqrt{1-w^2}}} \cdot \frac{b^2 e^{-0.5}}{\sqrt{1-w^2}} \cdot \int_{-1}^{\frac{x+\sqrt{1-w^2}}{w}} ((\frac{x-wz}{\sqrt{1-w^2}}+1)(z+1))^{-0.5}e^{-((1-\frac{w}{\sqrt{1-w^2}})z+1)/2} dz \\
 = & e^{-\frac{x}{2\sqrt{1-w^2}}} \cdot g(x),
 \end{align*}
 where the second to last equality involves a change of variable from $z$ to $\frac{z}{w}$, and the last equality combines every other term into $g(x)$.
 
 The leading exponential term dominates the decay rate of the density, while $g(x)$ is at most $O(x)$: the term $\frac{b^2 e^{-0.5}}{\sqrt{1-w^2}}\leq\sqrt{2}b<1$ is a fixed constant; the upper bound of the integral increases at $O(x)$; the polynomial term of the integral is $O(1)$; and the remaining exponential term in the integral satisfies
 \begin{align*}
 e^{-((1-\frac{w}{\sqrt{1-w^2}})z+1)/2} \leq 1
 \end{align*}
 for any fixed $w \in (0,\frac{1}{\sqrt{2}}]$. Therefore, the density of $wU+\sqrt{1-w^2}V$ decays at the rate $O(e^{-xc/2})$, for which $c=(1-w^2)^{-0.5} \in (1,\sqrt{2}]$. When we consider $m=2$ and let $w_2=w$ be the smaller weight, $w_1=\sqrt{1-w^2}$ becomes the larger weight so $c=1/w_{1}$.
\end{proof}

\begin{lemma}
\label{lem2}
Suppose $U$ is the centered chi-square distribution, and $V$ is an $m-1$ weighted summation of $U_i$ using the weights $\{w_{i} (1-w_{m}^2)^{-0.5}, i=1,\ldots,m-1\}$ for any $m>2$ and $w_{m} \in (0,\frac{1}{\sqrt{m}}]$. Assume the density of $V$ satisfies
\begin{align*}
 & f_{V}(x) = O(e^{-x c_{m-1}/2})
\end{align*}
where $c_{m-1}=\sqrt{1-w_{m}^2}/w_{1} \in [1, \sqrt{m-1}]$. Then the density $f_{w_{m}U+\sqrt{1-w_{m}^2} V}(x)$ satisfies   
\begin{align*}
 & f_{w_{m}U+\sqrt{1-w_{m}^2} V}(x) = O(e^{-x c_{m}/2})
\end{align*}
where $c_{m}=1/w_{1} \in [1, \sqrt{m}]$. 
\end{lemma}
\begin{proof}
The initial case corresponds to Lemma~\ref{lem1} with $c_{1}=1/w_{1} \in (1,\sqrt{2}]$. Moreover, $\{w_{i} (1-w_{m}^2)^{-0.5}, i=1,\ldots,m-1\}$ is always a valid weighting scheme for $m-1$ summation because 
\begin{align*}
\sum_{i=1}^{m-1}w_{i}^{2}/(1-w_{m}^2)=1.
\end{align*} 
The density of $V$ must be of the form
\begin{align*}
f_{V}(x) = e^{-x c_{m-1}/2} g(x),
\end{align*}
where $g(x)$ is a function that grows at most $o(e^{x c_{m-1}/2})$. 

In the following, we let $w=w_{m}$ to simplify the expression. Then the summation density is
\begin{align*}
 & f_{w U + \sqrt{1-w^2} V} (x) \\ 
 =& \int_{-w}^{+\infty}f_{\sqrt{1-w^2} V}(x-z)f_{wU}(z)dz \\
 =& \frac{b}{\sqrt{1-w^2}} \int_{-1}^{\frac{x+\sqrt{m(1-w^2)}}{w}} (z+1)^{-0.5}e^{-(\frac{c_{m-1}(x-wz)}{\sqrt{1-w^2}}+z+1)/2} g(\frac{x-wz}{\sqrt{1-w^2}}) dz \\
 =& e^{-\frac{xc_{m-1}}{2\sqrt{1-w^2}}}\cdot \frac{b}{\sqrt{1-w^2}} \cdot  \int_{-1}^{\frac{x+\sqrt{m(1-w^2)}}{w}} (z+1)^{-0.5}e^{-((1-\frac{wc_{m-1}}{\sqrt{1-w^2}})z+1)/2} g(\frac{x-wz}{\sqrt{1-w^2}}) dz.
 \end{align*}
 The only exponential term within the integral satisfies 
 \begin{align*}
 e^{-((1-\frac{wc_{m-1}}{\sqrt{1-w^2}})z+1)/2} \leq 1.
 \end{align*}
 This is because
 \begin{align*}
 \frac{wc_{m-1}}{\sqrt{1-w^2}} \leq \frac{c_{m-1}}{\sqrt{m-1}} \leq \frac{\sqrt{m}}{m-1} <1
 \end{align*}
 for any $m > 2$, so that $(1-\frac{wc_{m-1}}{\sqrt{1-w^2}})z+1 \geq 0$ when $z > -1$. Analyzing every other term in the same manner as the base case in the proof of Lemma~\ref{lem1}, we conclude that the density is dominated by the leading exponential term. Therefore, the density decays at $O(e^{-xc_{m}/2})$, where
 \begin{align*}
 c_{m} &=\frac{c_{m-1}}{\sqrt{1-w^2}} = 1/w_{1} \in (1,\sqrt{m}].
 \end{align*}
\end{proof}

\subsection{Proof of Theorem~\ref{cor1}}
\begin{proof}
From Theorem~\ref{thm3}, the decay rate of the summation density is $O(e^{-x c/2})$ with $c\geq 1$. It always decays faster than $U$ such that for sufficiently large $x$,
\begin{align*}
f_{\sum\limits_{i=1}^{m}w_i U_i}(x) \leq f_{U}(x)
\end{align*}
with equality if and only if $m=1$, leading to upper tail dominance for sufficiently small $\alpha$. Moreover, $\mc{N}(0,2)$ has the same mean and variance as them, with a density decay rate of $O(e^{-x^2/8})$ that is always faster than $O(e^{-xc/2})$. Therefore, there exists $\alpha>0$ such that $\mc{N}(0,2) \preceq_{\alpha} n\Dcor(\mathbf{X},\mathbf{Y}) \preceq_{\alpha} U$.
\end{proof}

\subsection{Proof of Theorem~\ref{thm4}}
\begin{proof}
First, the density of $U$ is
\begin{align*}
& f_{U} (x) = b (x+1)^{-0.5}e^{-(x+1)/2}
\end{align*}
where $b=2^{-0.5} \Gamma(0.5)^{-1} \approx 0.4$ is the constant from standard chi-square distribution of degree $1$. The domain of $U$ is $(-1,+\infty)$, and the density equals $0$ otherwise. 

When $w_i=\frac{1}{\sqrt{m}}$, 
\begin{align*}
W=\sum\limits_{i=1}^{m} w_i U_i = \sum\limits_{i=1}^{m} U_i / \sqrt{m} \sim \frac{\chi^{2}_{m}-m}{\sqrt{m}},
\end{align*}
whose density equals 
\begin{align*}
\frac{\sqrt{m}}{2^{\frac{m}{2}}\Gamma(m/2)}(\sqrt{m} x+m)^{\frac{m}{2}-1} e^{-\frac{\sqrt{m}x+m}{2}}.
\end{align*}
At fixed $m$, the density of $W$ decays exponentially at the rate $O(e^{-\sqrt{m}x/2})$. This matches Theorem~\ref{thm3}, and there must exist $x$ such that $f_{U}(x') \geq f_{W}(x')$ for all $x' \geq x$.

As the distribution is known, we can exactly compute the argument $x$ such that $F_{W}(x) \geq F_{U}(x)$, which is monotonically decreasing as $m$ increases. In particular, $x=2.7 < F_{U}^{-1}(0.95)$ when $m=2$; $x=2.5$ when $m=3$; $x=2.3$ when $m=10$; and $x=2$ when $m=1000$. Therefore $\alpha$ is at least $0.05$ regardless of $m$, and $W \preceq_{0.05} U$ always holds in the equal weight case. As $m \rightarrow \infty$, the validity level converges to $\alpha = 0.0875$ from the proof of Theorem~\ref{thm6}.
\end{proof}

\subsection{Proof of Theorem~\ref{thm5}}
\begin{proof}
This theorem follows from Lemma~\ref{lem3}: when $X$ and $Y$ are binary random variables, the centered sample matrix $\mathbf{H}\mathbf{D}^{\mathbf{X}}\mathbf{H}$ only has $1$ nonzero eigenvalue, so is $\mathbf{H}\mathbf{D}^{\mathbf{Y}}\mathbf{H}$. As a result, the eigenvalue products in Equation~\ref{eq1} satisfy $\lambda_1 \mu_1 >0$ and $\lambda_i \mu_j =0$ otherwise.

Once the eigenvalue products are normalized into $w_{ij}$ (see proof of Theorem~\ref{thm1}), it follows that $w_{11}=1$ and $w_{ij}=0$ otherwise, and $n \Dcor_{n}(\mathbf{X},\mathbf{Y}) \stackrel{D}{\rightarrow} U$.
\begin{lemma}
\label{lem3}
Suppose the sample data $\mathbf{X}$ has at most $m$ distinct values. Then the sample matrix $\mathbf{H}\mathbf{D}^{\mathbf{X}}\mathbf{H}$ has at most $m-1$ non-zero eigenvalues regardless of $n$.
\end{lemma}
This lemma can be argued via the sample matrix as follows:
First, as $\mathbf{H}$ is the centering matrix, the eigenvalues of $\mathbf{H}\mathbf{D}^{\mathbf{X}}\mathbf{H}$ always equal the eigenvalues of $\mathbf{D}^{\mathbf{X}}\mathbf{H}$. Next, observe that $det(\mathbf{H})=0$ and thus $det(\mathbf{D}^{\mathbf{X}}\mathbf{H})=0$, so there exists at least one zero eigenvalue and at most $n-1$ non-zero eigenvalues. The geometric multiplicity of the zero eigenvalue equals $n-rank(\mathbf{D}^{\mathbf{X}}\mathbf{H})$, so the total number of non-zero eigenvalues is $rank(\mathbf{D}^{\mathbf{X}}\mathbf{H})=rank(\mathbf{D}^{\mathbf{X}})-1$. When $X$ has $m$ distinct values, the distance matrix only has $m$ distinct rows and $rank(\mathbf{D}^{\mathbf{X}})=m$. Therefore the centered matrix has $m-1$ non-zero eigenvalues.
\end{proof}

\subsection{Proof of Theorem~\ref{thm6}}
\begin{proof}
As $p \rightarrow \infty$ and $X$ is continuous, the Euclidean distance matrix $\mathbf{D}^{\mathbf{X}}$ converges to $a(J-I)$, so the centered matrix converges to $a(-I)+a J/n$, where $a$ is a constant depending on the metric choice and marginal distribution $F_{X}$. Using Lemma~\ref{lem3}, $\mathbf{H}\mathbf{D}^{\mathbf{X}}\mathbf{H}/n$ has $1$ zero eigenvalue and $n-1$ non-zero eigenvalues that are asymptotically the same. Similarly for $\mathbf{H}\mathbf{D}^{\mathbf{Y}}\mathbf{H}/n$ when $q \rightarrow \infty$ and $Y$ is continuous.

This is essentially the asymptotic case of Theorem~\ref{thm4} as $m \rightarrow \infty$. By Lyapunov central limit theorem:
\begin{align*}
\sum\limits_{i=1}^{m} w_i U_i \stackrel{m\rightarrow \infty}{\rightarrow} \mc{N}(0,2).
\end{align*}
Evaluating the cumulative distribution of standard normal, it follows that $\mc{N}(0,2) \preceq_{\alpha} U$ at $\alpha = 0.0875$.
Therefore, when either $p$ or $q$ increases to infinity and $n$ also increases to infinity, the limiting null distribution becomes a normal distribution and satisfies $\mc{N}(0,2) \preceq_{0.0875} U$. 
\end{proof}

\subsection{Proof of Corollary~\ref{cor2}}
\begin{proof}
Given a test $z$ with a fixed distribution $Z$, it being always valid for testing independence at level $\alpha$ requires
\begin{align*}
U \preceq_{\alpha} Z
\end{align*}
because $U$ is the limiting null distribution of distance correlation between binary random variables by Theorem~\ref{thm5}. Thus, 
\begin{align*}
n \Dcor_{n}(\mathbf{X},\mathbf{Y}) \preceq_{\alpha} U \preceq_{\alpha} Z
\end{align*}
when testing independence between arbitrary random variables, and 
\begin{align*}
\beta_{\alpha}^{z} < \beta_{\alpha}^{\chi} \leq \beta_{\alpha}
\end{align*}
always holds. Therefore, the chi-square test is the most powerful test among all valid tests of distance correlation using known distributions.
\end{proof}

\section{Simulation Details}
\label{app_sim}
The simulations in Figure~\ref{fig1} and Figure~\ref{fig4} are generated via the exponential relationship:
\begin{align*}
\mbx[d] &\sim Uniform(-3,1) \mbox{ for each } d=1,\ldots,p\\
\mby &=\exp(\mbx) + 0.2*\epsilon,
\end{align*}
where $q=p$ and $\epsilon$ is an independent standard Cauchy random variable. 

The 1-dimensional sample data in Figure~\ref{fig2} are generated via the following:
\begin{itemize}
\item Linear $(\mbx,\mby)$:
\begin{align*}
\mbx &\sim Uniform(-1,1),\\
\mby &=\mbx + \epsilon.
\end{align*}
\item Quadratic $(\mbx,\mby)$:
\begin{align*}
\mbx &\sim Uniform(-1,1),\\
\mby&=\mbx^2 + 0.5\epsilon.
\end{align*}
\item Spiral $(\mbx,\mby)$: let $Z \sim \mc{N}(0,5)$, $\epsilon \sim \mc{N}(0, 1)$,
\begin{align*}
\mbx&=Z \cos(\pi Z),\\
\mby&= Z \sin(\pi Z) +0.4 \epsilon.
\end{align*}
\item Independent $(\mbx,\mby)$: let $Z \sim \mc{N}(0,1)$, $W \sim \mc{N}(0,1)$, $Z' \sim Bernoulli(0.5)$, $W' \sim Bernoulli(0.5)$, 
\begin{align*}
\mbx&=Z/3+2Z'-1,\\
\mby&=W/3+2W'-1.
\end{align*}
\end{itemize}
In all four simulations we have $p=q=1$, and $\epsilon$ is an independent standard normal random variable.

The increasing-dimensional simulations in Figure~\ref{fig3} are generated via 
\begin{itemize}
\item Equal variance:
\begin{align*}
\mbx[d] &\sim Uniform(-1,1) \mbox{ for each } d=1,\ldots,p\\
\mby &=\mbx[1];
\end{align*}
\item Minimal variance:
\begin{align*}
\mbx[d] &\sim Uniform(-1,1), \mbox{ for } d=1,\ldots,20,\\
\mbx[d] &\sim \frac{1}{p} \cdot Uniform(-1,1), \mbox{ for } d=21,\ldots,p,\\
\mby &=\mbx[1];
\end{align*}
\item Dependent coordinate:
\begin{align*}
\mbx[1] &\sim Uniform(-1,1),\\
\mbx[d] &= 0.5 \mbx[d-1]+Uniform(-0.5,0.5), \mbox{ for } d=2,\ldots,p,\\
\mby &=\sum\limits_{d=1}^{p} \mbx^2[d];
\end{align*}
\item Varying marginal:
\begin{align*}
\mbx[d] &\sim \mc{N}(d,d), \mbox{ for } d=1,\ldots,p,\\
\mby &=\mbx[1].
\end{align*}
\end{itemize}
In all four simulations we have $q=1$, sample size $n$ set to $20$, $100$, $50$, $100$ respectively, and $p$ increases from $100$ to $1000$ in the first two simulations and $2$ to $20$ in the latter two simulations.
\end{document}